\documentclass[dvips]{imsart}

\usepackage{amsthm,amsmath}
\RequirePackage[numbers]{natbib}
\RequirePackage[colorlinks,citecolor=blue,urlcolor=blue]{hyperref}

\pubyear{0000}
\volume{0}
\firstpage{0}
\lastpage{0}
\startlocaldefs
\numberwithin{equation}{section}
\newtheorem{theorem}{Theorem}[section]
\newtheorem{corollary}{Corollary}[section]
\newtheorem{lemma}{Lemma}[section]
\newcommand{ \rmL} {{\rm{L}}}

\newcommand{\bbeta}{\mbox{\boldmath{$\beta$}}}
\newcommand{\balpha}{\mbox{\boldmath{$\alpha$}}}
\endlocaldefs

\usepackage{graphicx}

\begin{document}

\begin{frontmatter}

% "Title of the Paper"
\title{The theory and application of penalized
 methods
 \\
 or
 \\
  Reproducing
Kernel Hilbert Spaces made easy}
\date{November 4 2011}
\thankstext{t1}{The author wishes to acknowledge the support of the Natural Sciences
and Engineering Research Council of Canada through grant A7969.}
\runtitle{RKHS Made Easy}

% indicate corresponding author with \corref{}
% \author{\fnms{John} \snm{Smith}\thanksref{t2}\corref{}\ead[label=e1]{smith@foo.com}\ead%[label=e2,url]{www.foo.com}}
% \thankstext{t2}

\author{\fnms{Nancy} \snm{Heckman}\ead[label=e1]{nancy@stat.ubc.ca}\protect\thanksref{t1}}

 \address{Department of Statistics\\ The University of British Columbia\\
 Vancouver BC Canada\\ \printead{e1}}

\runauthor{N Heckman}

\begin{abstract}
The popular cubic smoothing spline estimate of a regression function arises as the 
minimizer of the penalized sum of squares $\sum_j(Y_j - \mu(t _j))^2 +  \lambda \int_a^b [ \mu''(t)]^2~dt$,  
where  the data are $t_j,Y_j$, $j=1,\ldots, n$.  The minimization is taken over an 
infinite-dimensional function space, the space of all functions with square integrable 
second derivatives.  But the calculations can be carried out in a finite-dimensional space.
The reduction from minimizing over an infinite dimensional space to minimizing over a finite 
dimensional space occurs for more general objective functions:  the data  may be related to 
the function $\mu$  in another way,  the sum of squares may be replaced by a more suitable 
expression, or the penalty, $\int_a^b [ \mu''(t)]^2~dt$,  might take  a different form.  
This paper reviews the Reproducing Kernel Hilbert Space structure that provides
a finite-dimensional solution for a general minimization problem.   
Particular attention is paid to penalties based on linear differential operators.  In this case, 
one can sometimes easily calculate the minimizer explicitly, using Green's functions.
 \end{abstract}

\begin{keyword}[class=AMS]
\kwd[Primary ]{62G99}
\kwd{46E22}
\kwd[; secondary ]{62G08}
\end{keyword}

\begin{keyword}
\kwd{Penalized likelihood}
\kwd{Reproducing Kernel Hilbert Space}
\kwd{Splines}
\end{keyword}
\tableofcontents

\end{frontmatter}

%%%%%%%%%%%%%%%%%%%%%%%%%%%%%%%%%%%%%%%%%%%%

\section{Introduction}

A Reproducing Kernel Hilbert Space (RKHS) provides a practical and elegant structure to
solve optimization problems in function spaces.    
This article considers the use of an RKHS to analyze the data 
$Y_1,\ldots Y_n \in \Re$ and $t_1,\ldots, t_n \in \Re^p$.  The distribution of the $Y_i$'s depends
on $\mu$, a function of $t \in \Re^p$, which is usually assumed to be smooth.   The goal is to  find  $\mu$ 
in a specified function space ${\cal{H}}$ 
to minimize
\begin{equation}
\label{eq:general}
G(t_1, \ldots, t_n, Y_1,\ldots,Y_n, F_1(\mu),\ldots, F_n(\mu))  + \lambda P(\mu)
\end{equation}
where   $G$ and the $F_j$'s are known,  $P$ is a known penalty on $\mu$,  and $\lambda$ serves
to balance the importance between $G$ and $P$.  
Typically, $F_j(\mu) = \mu(t_j)$ and $P(\mu)$ is based on derivatives of $\mu$.   Some results here will concern
general $P$ and $t_j\in \Re^p$, $p \ge 1$,  and some more extensive results will concern  $t_j \in [a,b] \subset \Re$ 
and $P$ generated from a differential operator $\rmL$:
\begin{equation}
\label{eq:differentialP}
P(\mu) = 
\int_a^b [ (\rmL\mu)(t)]^2 ~dt
{\rm{~~~where ~~~}}
(\rmL \mu)(t) = \mu^{(m)}(t) + \sum_{j=0}^{m-1} w_j(t) \mu^{(j)}(t)
\end{equation}
with $w_j$ real-valued and continuous.    For this type of penalty, we restrict $\mu$ to lie in
the space
\[
{\cal{H}}^m[a,b]=
\{f : [a; b] \to \Re  : \mu^{(j)}; j = 0; \cdots , m -  1 {\rm{~ are~ absolutely~ continuous}}
\]
\[
{\rm{and~}} \int_a^b [ \mu^{(m)}(t)]^2 ~dt < \infty  \}.
\]
Note that,  for all $\mu \in {\cal{H}}^m[a,b] $, $\int_a^b [ (\rmL\mu)(t)]^2 ~dt$ is well defined: 
L$\mu(t)$ exists almost everywhere $t$ and L$\mu$ is square
integrable, since the $\omega_j$'s are continuous and $[a, b]$ is finite.

The most well-known case of (\ref{eq:general}) 
occurs in regression analysis, when we seek  the regression function $\mu \in {\cal{H}}^2[a,b]$ to minimize
\begin{equation}
\label{eq:splines}
\sum_j [Y_j - \mu(t _j)]^2 +  \lambda \int_a^b [ \mu''(t)]^2~dt.
\end{equation}
The  minimizing $\mu$ is a cubic smoothing spline, a popular regression function estimate.
The non-negative smoothing parameter 
$\lambda$  balances  the minimizing $\mu$'s  fit to the data (via minimizing
$ \sum_j[Y_j - \mu(t _j)]^2 $) with  its
closeness to a straight line (achieved when
$  \int_a^b [ \mu''(t)]^2~dt = 0$).  
The value of $\lambda$ is typically chosen ``by eye" -- by examining the resulting estimates of $\mu$,
or by some automatic data-driven method such as cross-validation.
 See, for instance, Wahba \cite{Wahba}, Eubank \cite{Eubank} or Green and Silverman \cite{Green.and.Silverman}.

To extend (\ref{eq:splines})  to (\ref{eq:general}), we can consider a first  
 term other than a sum of squares,   functionals other than
$F_j(\mu) = \mu(t_j)$ and a differential operator other than the second derivative operator.  
Examples of these variations are given in Section \ref{sec:examples}.
Section  \ref{sec:general} contains the reduction of (\ref{eq:general}) to a finite dimensional optimization problem.  
Section \ref{sec:Bayes} relates the minimizer of (\ref{eq:general}) to a Bayes estimate.  Sections \ref{sec:cubic} and
\ref{sec:general_differential} contain
results and algorithms for minimizing (\ref{eq:general}) with $P$ as in (\ref{eq:differentialP}), with Section \ref{sec:cubic}
containing the ``warm-up" of the cubic smoothing spline result for minimizing (\ref{eq:splines}) and Section 
\ref{sec:general_differential}
containing the general case.  The Appendix contains pertinent results from the theory of solutions of
differential equations.
 
The material contained here is, for the most part,
not original.  The material is
 drawn from many sources: from statistical and machine learning literature, from the theory of differential equations, from numerical analysis, and from functional analysis.    
 The purpose of this paper is to collect this diverse material in one article and to
 present it in an easily accessible form,
 to show the richness of statistical problems that involve minimizing
(\ref{eq:general}) and to explain  the
theory and provide easy to follow algorithms  for minimizing (\ref{eq:general}).  
A briefer review of  RKHS's can be found in Wahba \cite{Wahba.review}.
%%%%%%%%%%%%%%%%%%%%%%%%%%%%%%%%%%%%%%%%%%%%

\section{Examples}
\label{sec:examples}
%%%%%%%%%%%%%%%%%%%%%%%%%%%%%%%%%%%%%%%%%%%%

\subsection{Penalized likelihoods with $F_j(f) = f(t_j)$}

Most statistical applications that lead to minimizing (\ref{eq:general}) 
have the first term in  (\ref{eq:general})  equal to a negative log likelihood.  In these cases,
 the $\mu$ that minimizes (\ref{eq:general}) is called a penalized likelihood estimate of $\mu$.
Indeed, (\ref{eq:splines}) yields a penalized likelihood estimator:
the sum of squares  arises from a likelihood 
by assuming that $Y_1,\ldots,Y_n$ are independent normally distributed with
the mean of $Y_j$ equal to $\mu(t_j)$ and the variance equal to $\sigma^2$.  Then $-2 \times$ the log likelihood is simply
\[
n \log (\sigma^2) +  \frac{1}{\sigma^2}  \sum(Y_j - \mu(t_j))^2.
\]
A penalized likelihood estimate of $\mu$ with penalty $P(\mu)$  minimizes
\[
n \log (\sigma^2) +  \frac{1}{\sigma^2}  \sum(Y_j - \mu(t_j))^2
+ \lambda^* P(\mu)
\]
\[ ~~~~~~~~~~
=
\frac{1}{\sigma^2} 
\left[   {\sigma^2} 
n \log (\sigma^2) + \sum(Y_j - \mu(t_j))^2
+  {\lambda^*}{\sigma^2}  P(\mu) \right].
\]
Thus, for a given $\sigma^2$, the penalized likelihood estimate of
$\mu$ minimizes (\ref{eq:general})  with $\lambda = \lambda^* \sigma^2$.
If the $Y_j$'s are not independent but the vector $(Y_1,\ldots, Y_n)'$ has covariance matrix $\sigma^2  \Sigma$, 
then we would replace the 
sum of squares with $\sum_{j,k} [Y_j - \mu(t_j)] ~\Sigma^{-1}[j,k] ~[Y_k - \mu(t_k)]$.

Another likelihood, important in classification, is based on data  $Y_j = 1$ or $-1$  with probabilities $p(t_j)$ and $1-p(t_j)$,
respectively.  
Thus  
\[
 {\rm{~the~ log~ likelihood~}}
= 
{\sum}_j      \frac{1+Y_j}{2} ~ \log p(t_j)  +  \frac{1-Y_j}{2}~ \log [1-p(t_j)  ].
\]
To avoid placing inequality constraints on the function of interest,
 we reparameterize by setting $\mu(t)= \log[ p(t)/( 1 - p(t))]$ or equivalently $p(t) = \exp(\mu(t))/[ 1 + \exp(\mu(t))].$ 
 This reparameterization yields
 \begin{equation}
\label{eq:logit}
 {\rm{~the~ log~ likelihood~}}
=  \sum_j  \frac{1+Y_j}{2} \log \frac{ \exp(\mu(t_j))}{ 1 + \exp(\mu(t_j))}
 \ +  \frac{1-Y_j}{2} \log \frac{ 1}{ 1 + \exp(\mu(t_j))} .
\end{equation}

%%%%%%%%%%%%%%%%%%%%%%%%%%%%%%%%%%%%%%%%%%%%

\subsection{$F_j$'s based on integrals}
\label{sec:integral_F}

While $F_j(\mu)=\mu(t_j)$ is common, $F_j(\mu)$ is sometimes chosen to involve an integral of $\mu$, specifically,
 $F_j(\mu) =  \int_a^ b H(s, t_j) \mu(s) ds $,
with $ H$  known.  See  Wahba  \cite{Wahba}. 

 Li \cite{Li} and  Bacchetti  {\em{et al.}}~\cite{Bacchetti.Segal.Hessol.and.Jewell} 
 used (\ref{eq:general}) to estimate $\mu(t)$, the HIV infection rate at time $t$, based on data,  $Y_j$,
the number of new AIDS cases diagnosed in time period $(t_{j-1},t_j]$.    The expected
value of $Y_j$ depends not only on $\mu(t_j)$, but also
on $\mu(t)$ for values of $t\le t_j$.   This dependence involves
the distribution of the time of progress from HIV infection to AIDS diagnosis, which is estimated from
cohort studies.  Letting ${\cal{F}}(t|s)$
denote the probability that AIDS has developed by time $t$ given HIV infection occurred at time $s$,
\[
{\rm{E}} \left( \sum_1^j Y_i\right)  = \int_{s=0}^{t_j}  \mu(s)  {\cal{F}}(t_j |s)  ~  ds \equiv F_j(\mu).
\] 
Thus we could define the first term in (\ref{eq:general}) as a negative log likelihood assuming the $Y_j$'s are independent
 Poisson counts with E$(Y_j) = F_j(\mu) - F_{j-1}(\mu)$. 
  Or we could take the computationally simpler approach by setting the first term in (\ref{eq:general})  equal to
\[
\sum_1^n \bigg{\{}Y_j - \big[ F_j(\mu) - F_{j-1}(\mu)\big]  \bigg{\}}^2.
\]
Both  Li  \cite{Li}  and Bacchetti {\em{et al.}}~\cite{Bacchetti.Segal.Hessol.and.Jewell}  use this simpler
approach, with the former using penalty $P(\mu)=\int (\mu'')^2$ while 
the latter used  a discretized version of
 $\int (\mu'')^2$.

In a non-regression setting, Nychka  {\em{et al.}}~\cite{Nychka.Wahba.Goldfarb.and.Pugh}   estimated
 the distribution of the volumes of tumours in livers  by using  data  from 
cross-sectional slices of the livers. The authors modelled tumours as spheres and so
cross-sections were circles.  They
estimated $\mu$, the probability density of the spheres' radii, using an integral to relate
the radius of a sphere to the radius of
a random slice of the sphere.    Their estimation criterion was the minimization of an expression of the form
(\ref{eq:general}) with $F_j$ using that integral and with $P(\mu) = \int(\mu'')^2$.
 
%%%%%%%%%%%%%%%%%%%%%%%%%%%%%%%%%%%%%%%%%%%%

\subsection{Support vector machines}

Support vector machines are a classification tool, with classification rules built from data $Y_i \in \{ -1, 1\}$,
$t_i \in \Re^p$ (see, for instance, Hastie {\em{et al.}}~\cite{Hastie.Tib}).  The goal is to find a function $\mu$
for classifying:  classify $Y_i$ as $1$ if and only if $\mu(t_i)  > 0$.   We see that $Y_i$ is misclassified by this rule
if and only if $Y_i \mu(t_i) $ is positive.  Thus, it is common
to find $\mu$   to minimize $\sum_i$sign$[Y_i \mu(t_i)]$ subject to
some penalty for rough $\mu$: that is, to find $\mu$ to minimize
\[
\sum_i\text{sign}[Y_i \mu(t_i)] + \lambda P(\mu).
\]
This can be made more general by minimizing
\[
\sum_j H[ Y_i \mu(t_j)] + \lambda P(\mu)
\]
for a known non-decreasing function $H$.  
The function $H(x) = $ sign$(x)$ is not continuous at 0, which can make minimization challenging.   
To avoid this problem, Wahba \cite{Wahba.NIPS} proposed using ``softer"  $H$ functions, such
as  $H(x) = \ln[ 1 + \exp (-x)]$.  This function is not only continuous, but is differentiable and convex.   Wahba
\cite{Wahba.NIPS} showed that this $H$ corresponds to a negative log likelihood.  Specifically, she showed that
 the log likelihood in  (\ref{eq:logit}) is equal to 
  $- \sum \log  \{ 1 + \exp\left[  - Y_j \mu(t_j)\right] \}$.

%%%%%%%%%%%%%%%%%%%%%%%%%%%%%%%%%%%%%%%%%%%%

\subsection{Using different differential operators in the penalty}

Ansley, Kohn, and Wong  \cite{Ansley.Kohn.and.Wong} and
Heckman and Ramsay  \cite{Heckman.and.Ramsay} demonstrated the usefulness of appropriate
choices of $\rmL$ in the penalty $P(\mu) = \int (\rmL\mu)^2$.   
For instance, Heckman and Ramsay 
compared two estimates of a regression function for the incidence of melanoma in males. The data, described in Andrews
and Herzberg \cite{Andrews.and.Herzberg}, are from the Connecticut Tumour Registry, for the years
1936 to 1972.  The data show a roughly periodic trend superimposed on an increasing trend.
A cubic smoothing spline, the minimizer of (\ref{eq:splines}), tracks the data fairly well, but slightly dampens
the periodic component.  This dampening does not occur with Heckman and Ramsay's preferred estimate,
the estimate that 
 minimizes a modified version of (\ref{eq:splines}) but with the  penalty $\int[ \mu''(t)]^2 ~dt$  replaced
 by
the penalty $\int[  \mu^{(4)}(t) + \omega^2 \mu''(t)]^2 ~ dt$ with $\omega= 0.58.$
The differential operator L$=$D$^4 + \omega^2$D$^2$ was chosen since it  places no penalty
on  functions of the form  $\mu(t) = \alpha_1 + \alpha_2 t+
 \alpha_3 \cos \omega t+  \alpha_4 \sin \omega t$:
  such functions are exactly the functions satisfying L$\mu \equiv 0$ and
form a popular parametric model for fitting melanoma data. The value of $\omega$
was chosen by a nonlinear least squares fit to this parametric model.

The use of appropriate differential operators in the penalty has been further developed 
in the field of Dynamic Analysis.    See, for instance, Ramsay  {\em{et al.}}~\cite{Ramsay.JRSSB}.
These authors use  differential operators  equal to those used by subject area researchers, who typically
work in the finite dimensional space defined by solutions of L$\mu \equiv 0$.

%%%%%%%%%%%%%%%%%%%%%%%%%%%%%%%%%%%%%%%%%
 \section{Results for the general minimization problem}
 \label{sec:general}
 
 This section contains some background on Reproducing Kernel Hilbert Spaces and shows how to
 use Reproducing Kernel Hilbert Space structure to reduce the minimization of (\ref{eq:general}) to minimization
 over  a finite-dimensional function space (see Theorem \ref{thm:finite}).  
  Whether or not the minimizer exists  can be determined by studying the finite-dimensional version.
    While a complete review of Hilbert spaces
 is beyond the scope of this article, a few definitions may help the reader. Further background on Hilbert spaces
 can be found in any standard functional analysis textbook, such as Kolmogorov and Fomin
 \cite{Kolmogorov}  or Kreyszig \cite{Kreyszig}.
 For a condensed exposition of the necessary Hilbert space theory,
 see, for instance,  Wahba \cite{Wahba}, \cite{Wahba.review} or the appendix of Thompson and Tapia \cite{Tapia.and.Thompson}.  We will only consider Hilbert spaces over $\Re$. 
 
 Consider ${\cal{H}}$, a collection of functions from ${\cal{T}} \subseteq \Re^p$ to $\Re$.  Suppose that
 ${\cal{H}}$ is a vector space over $\Re$ with inner product $< \cdot,\cdot>$.   The inner product  induces a norm
 on ${\cal{H}}$, namely $||f|| = [ <f,f>]^{1/2}$.  The existence of a norm allows us to define limits of sequences in ${\cal{H}}$ and continuity
 of functions with arguments in ${\cal{H}}$.  The vector space 
 ${\cal{H}}$ is a {\em{Hilbert space}} if it is complete with respect to this norm, that is, if any Cauchy sequence
 in ${\cal{H}}$ converges to an element of ${\cal{H}}$.  
 
  A {\em{linear functional}} $F$ is a function
 from a Hilbert space ${\cal{H}}$ to the reals satisfying $F(\alpha f + \beta g) = \alpha F(f) + \beta F(g)$
 for all $\alpha, \beta \in \Re$ and all $f, g \in {\cal{H}}$.
  The Riesz Representation Theorem states that a linear functional $F$ is continuous on 
  ${\cal{H}}$ if and only if there exists $\eta \in {\cal{H}}$ such that $<\eta,f> = F(f)$ for all
$f \in {\cal{H}}$.  The function $\eta$ is called the representer of $F$.   

The Hilbert space
${\cal{H}}$  is a {\em{Reproducing Kernel Hilbert
Space}}  if and only if,  for all $t \in {\cal{T}}$,
the linear functional $F_t(f ) \equiv f (t)$ is continuous, that is, if and only if,
for all $t \in {\cal{T}}$, there exists $R_t \in {\cal{H}}$ such that $< R_t,f> = f(t)$ for all $f \in {\cal{H}}$.  
Noting that the collection of $R_t$'s, $t \in {\cal{T}}$, defines a bivariate function $R$, namely  $R(s,t) \equiv R_t(s)$, we see
that  
${\cal{H}}$  is a  Reproducing Kernel Hilbert
Space  if and only if there exists a
bivariate function $R$ defined on ${\cal{T}} \times {\cal{T}}$ such that
$<R(\cdot , t) , f >= f(t)$ for all $ f \in  {\cal{ H }}$  and all $t \in {\cal{T}}$.  The function
$R$ is called the reproducing kernel of ${\cal{H}}$.

One can show that the reproducing kernel is  symmetric in its arguments, as follows.  To aid the proof, use the notation
that $R_t(s)=R(s,t)$ and $R_s(t)=R(t,s)$.  By the reproducing properties of $R_t$ and $R_s$, 
$< R_t, R_s> =R_s(t)$ and $< R_s, R_t> = R_t(s)$.  But the inner product is symmetric, that is
$< R_t, R_s> = < R_s, R_t> $.  So $R_s(t)=R_t(s)$.
 
 To give  the form of the finite-dimensional minimizer of (\ref{eq:general}), we assume that the following conditions hold.  
   \begin{itemize}
 \item[(C.1)]
 \label{condition:direct_sum}
 There are ${\cal{H}}_0$ and ${\cal{H}}_1$, linear subspaces of ${\cal{H}}$,
 with ${\cal{H}}_1$ the orthogonal complement of ${\cal{H}}_0$.
   \item[(C.2)]
 \label {condition:H0}
  ${\cal{H}}_0 $ is of dimension $m < \infty$, with basis $u_1,\ldots,u_m$.
  If $m=0$,  take ${\cal{H}}_0$ equal to the empty set and 
  ${\cal{H}}_1 = {\cal{H}}$.
   \item[(C.3)] 
  \label{condition:rk} 
  There exists $R_0 \in {\cal{H}}_0$ and $ R_1 \in {\cal{H}}_1$ such that $R_i$ is a reproducing kernel
  for ${\cal{H}}_i$, in the sense that 
 $< R_i(\cdot,t), \mu> = \mu(t)$ for all $\mu \in {\cal{H}}_i$, $i=0,1$.
  \end{itemize}
  Since ${\cal{H}}_0$ is finite dimensional, it is closed.  The orthogonal complement of a subspace is always closed.
  Thus Condition (C.1) implies that any   $\mu \in {\cal{H}}$ can be written as
  $\mu = \mu_0 + \mu_1$ for some $\mu_0 \in {\cal{H}}_0$ and $\mu_1 \in {\cal{H}}_1$ and that
  $< \mu_0, \mu_1> = 0$.   This is often written as ${\cal{H}} = {\cal{H}}_0 \oplus {\cal{H}}_1$.
Note that Conditions (C.1), (C.2)  and (C.3) imply that $R \equiv R_0 + R_1$ is
a reproducing kernel for ${\cal{H}}$.

We require one more condition, relating the penalty $P$ to the partition of ${\cal{H}}$.
\begin{itemize}
  \item[(C.4)]
  \label{condition:P}  Write $\mu = \mu_0 + \mu_1$, with $\mu_i \in {\cal{H}}_i$.  Then 
 $P(\mu) = < \mu_1, \mu_1>$.
 \end{itemize}

\vskip 10pt
\noindent
\begin{theorem}
\label{thm:finite}
 Suppose that conditions (C.1) through (C.4) hold
and that
$F_1, \ldots, F_n$ are  continuous linear functionals on ${\cal{H}}$.
Let $\eta_{j1} (t) = F_j(R_1(\cdot,t))$, that is, $F_j$ applied to 
the function $R_1$ considered as a function of $s$, with $t$ fixed.
Then to minimize   (\ref{eq:general}), it is necessary and sufficient to find 
\[
{\mu}(t) \equiv  \mu_0(t) + \mu_{11}(t) \equiv \sum_1^m \alpha_j u_j (t)  + \sum_1^n  \beta_j \eta_{j1}(t)
\]
where the $\alpha_j$'s and $\beta_j$'s  minimize
\[
G(t_1,\ldots,t_n,Y_1, \ldots, Y_n, F_1(\mu_0 + \mu_{11}),\ldots, F_n(\mu_0 + \mu_{11}))  + \lambda 
\bbeta' {K} \bbeta.
\] 
Here $\bbeta = (\beta_1,\ldots,\beta_n)'$ and the matrix  ${K}$ is symmetric and non-negative definite, with
${K}[j,k]= F_j(\eta_{k1})$.   If $F_j(f) = f(t_j)$ and $F_k(f) = f(t_k)$, then 
$\eta_{1j}(t) =  R_1(t_j,t) $, $\eta_{1k}(t) =  R_1(t_k,t) $ and 
 ${K}[j,k]= R_1(t_j,t_k) $.
\end{theorem}
\begin{proof} 
By the Riesz Representation Theorem, there exists  a representer
$\eta_j \in {\cal{H}}$ such that $<\eta_j, \mu> = F_j(\mu)$
for all $\mu \in {\cal{H}}$.  Applying the Riesz Representation Theorem to the subspaces ${\cal{H}}_0$ and
${\cal{H}}_1$, which can be considered as Hilbert spaces in their own rights,
there exists $\eta_{j0} \in {\cal{H}}_0$ and $\eta_{j1} ^*\in {\cal{H}}_1$,
 representers of $F_j$ in the sense that
 $< \eta_{j0}, \mu> = F_j(\mu)$ for all $\mu \in {\cal{H}}_0$ and
 $< \eta_{j1}^*, \mu> = F_j(\mu)$ for all $\mu \in {\cal{H}}_1$.  One easily shows that this $\eta_{j1}^*$ is 
 equal to $\eta_{j1}$, as defined in the statement of the Theorem:
by the definition of the representer of $F_j$,
  $\eta_{j1}^*$ must satisfy $F_j(R_1(\cdot , t)) =
 <\eta_{j1}^*, R_1(\cdot , t)>$.  But, by the reproducing quality of $R_1$,  $ <\eta_{j1}^*, R_1(\cdot , t)>= \eta_{j1}^*(t)$.  
So $\eta_{j1}^* = \eta_{j1}$. One also easily shows that
 \[
 \eta_j = \eta_{j0} + \eta_{j1}.
 \]
  
 We  use the $\eta_{j1}$'s to partition ${\cal{H}}_1$ as follows.  Let
${\cal{H}}_{11}$ be the finite dimensional subspace of ${\cal{H}}_1$ spanned by $\eta_{j1}, j = 1,\cdots, n$, and let
${\cal{H}}_{12}$ be the orthogonal complement of ${\cal{H}}_{11}$ in ${\cal{H}}_1$.
Then
$
{\cal{H}}  =  {\cal{H}}_0  \oplus {\cal{H}}_{11} \oplus {\cal{H}}_{12}
$
and  so any $\mu \in {\cal{H}}$ can
be written as
\[
\mu = \mu_0 + \mu_{11} + \mu_{12} \quad  {\rm{with}}~~ \mu_0 \in {\cal{H}}_0 ~~ {\rm{and}} ~~
 \mu_{1k} \in {\cal{H}}_{1k}, k = 1, 2.
\]

We now show that  any minimizer of (\ref{eq:general}) must have $\mu_{12}\equiv 0$.
Let $\mu$ be any element of ${\cal{H}}$. 
Since $\eta_j$ is
the representer of $F_j$ and $\mu_{12}$ is orthogonal to $\eta_j$,
\[
F_j (\mu) = <\eta_j, \mu> = <\eta_j , \mu_0 + \mu_{11} + \mu_{12}> = <\eta_j, \mu_0 + \mu_{11}> = F_j(\mu_0 + \mu_{11}).
\]
Therefore, $\mu_{12}$ is irrelevant in computing the first term in (\ref{eq:general}).
To study the second term in (\ref{eq:general}), by (C.4) and the orthogonality of $\mu_{11}$ and
$\mu_{12}$,
\[
P(\mu) = <\mu_1, \mu_1> = <\mu_{11}, \mu_{11}> + <\mu_{12}, \mu_{12}>.
\]

Therefore, we want to find $\mu_0 \in {\cal{H}}_0$, $\mu_{11}\in {\cal{H}}_{11}$ and $\mu_{12}\in {\cal{H}}_{12}$ to minimize
\[
G(t_1,\ldots,t_n,Y_1, \ldots,Y_n, F_1(\mu_0 + \mu_{11}),\ldots, F_n(\mu_0 + \mu_{11}))  + \lambda 
\left[<\mu_{11}, \mu_{11}>+<\mu_{12}, \mu_{12}>\right].
\]
Clearly, 
we should take $\mu_{12}$ to be the zero function and so any minimizer of (\ref{eq:general}) is of the form
\begin{eqnarray}
{\mu} (t) &=& \mu_0(t) + \mu_{11}(t)
\nonumber \\
&=&  \sum_1^m \alpha_j u_j (t)  + \sum_1^n  \beta_j \eta_{j1}(t).
\nonumber \end{eqnarray}

Now consider rewriting $P(\mu)$ as $\bbeta' {K} \bbeta$:
$P(\mu) =   <\mu_{11},\mu_{11}> = \sum_{j,k} \beta_j \beta_k < \eta_{j1},\eta_{k1}> \equiv
\bbeta 'K^* \bbeta$ for $K^*$ symmetric and non-negative definite.
To show that $K^*[j,k]= F_j(\eta_{k1})$, use  the fact that $\eta_{j1}$ is the representer of 
$F_j$ in  ${\cal{H}}_1$, that is, that $< \eta_{j1},f> = F_j(f)$ for all $f \in {\cal{H}}_1$.  Applying
 this to $f = \eta_{k1}$ yields the desired result, that $< \eta_{j1},\eta_{k1} >= F_j(\eta_{k1})$.
 
Consider the case that $F_j(f)=f(t_j)$ and $F_k(f) = f(t_k)$.   Then $\eta_{1j}(t) = F_j(R_1(\cdot,t) )= R_1(t_j,t)$, 
$\eta_{1k}(t) =  R_1(t_k,t)$, and $K[j,k] = F_j(\eta_{k1}) = R_1(t_k,t_j) = R_1(t_j,t_k)$ by symmetry of $R_1$.  
\end{proof}

\vskip 10pt

The proof of the following Corollary is immediate, by taking $m=0$ in (C.2).

\begin{corollary}
Suppose that ${\cal{H}}$ is an RKHS with inner product $<\cdot,\cdot>$ and reproducing kernel $R$.  
In (\ref{eq:general}), suppose that $P(\mu) = <\mu,\mu>$ and assume that the $F_j$'s are continuous
linear functionals.  Then the minimizer of
(\ref{eq:general}) is of the form
\[
\mu(t) = \sum_1^n  \beta_j F_j (R_1(\cdot,t)).
\]
\end{corollary}
%%%%%%%%%%%%%%%%%%%%%%%%%%%%%%%%%%%%%%%%%

\section{A Bayesian connection}
\label{sec:Bayes} 

Sometimes, the minimizer of (\ref{eq:general}) is related to a Bayes estimate of $\mu$. 
In the Bayes formulation,
 $Y_j = \mu(t_j) + \epsilon_j$ where the $\epsilon_j$'s are independent normal random variables
with  zero means and variances equal to  $\sigma^2$.   
The function   $\mu$ is  the realization of a stochastic process and is independent of the $\epsilon_j$'s.

The connection between $\hat{\mu}$, 
  the minimizer of  $\sum [Y_j - \mu(t_j)]^2  + \lambda \int (L\mu)^2$,  and   a Bayes estimate
 of $\mu$
was first given by  Kimeldorf  and Wahba \cite{Kimeldorf.Wahba}
for the case that L$\mu = \mu^{(m)}$.  The result was generalized
to  L's as in (\ref{eq:differentialP})  by Kohn and Ansley \cite{Kohn.Ansley}. 
The function   $\mu$  is defined on $\Re$ and is 
generated by the stochastic differential equation
L$\mu(t) ~dt = \sigma \sqrt{\lambda} ~dW(t)$ where $W$ is a mean zero Wiener process on $[a,b]$ with 
var$(W(t))=t$.  Assume that $\mu$
satisfies the initial conditions:
$\mu(a), \mu'(a),\ldots,\mu^{(m-1)}(a) $ are independent normal random variables with 
zero means   and variances equal to $k$.
Let $\hat{\mu}_{k}(t)$ be the posterior mean of $\mu(t)$ given $Y_1,\ldots,Y_n$. Then
Kimeldorf and Wahba
\cite{Kimeldorf.Wahba}
 and Kohn and Ansley \cite{Kohn.Ansley} show that 
$\hat{\mu}(t) = \lim_{k \to \infty} \hat{\mu}_{k}(t)$.

Another  Bayes connection arises in Gaussian process regression, a tool of machine learning (see, for instance,
Rasmussen and Williams \cite{Rasmussen}).  
Consider  $\mu$ defined on 
 $ A \subseteq \Re^p$, with $\mu$ 
 the realization of a mean zero stochastic process with covariance function $S$.  Let $\hat{\mu}_B$
be the pointwise Bayes estimate of $\mu$:
\[
\hat{\mu}_B(t) =  {\rm{E}}( \mu(t) | Y_1,\ldots, Y_n) = 
S(t,{\bf{t}})~\left[\sigma^ 2 {\rm{I} } + S({\bf{t}}, {\bf{t}}) \right]^{-1} {\bf{Y}}
\]
where
$S(t, \bf{t})'$ is  an $n$-vector with $j$th entry $S(t, t_j)$,
$S(\bf{t}, \bf{t})$ is the $n \times n$ matrix with $jk$th entry
$S(t_j, t_k)$ and  ${\bf{Y}} = (Y_1,\ldots,Y_n)'$.
Then, as shown below, for an
appropriately defined Reproducing Kernel Hilbert Space ${\cal{H}}_S$ with reproducing kernel $S$, the Bayes
estimate of $\mu$ is  equal to
\begin{equation}
\label{eq:Gaussian.Bayes}
\arg \min_{\mu \in {\cal{H}}_S} ~~ \sum_{j=1}^n [Y_j -  \mu(t_j)]^2 + {\sigma}^2  <\mu,\mu>.
\end{equation}
The existence of the space  ${\cal{H}_S}$ with reproducing kernel $S$ is given by the
Moore-Aronszajn Theorem (Aronszajn \cite{Aronszain}). 
The space is defined by constructing finite-dimensional spaces:  fix $J >0$ and $t_1, \ldots, t_J \in A$
and consider the finite dimensional linear space of functions,
${\cal{H}}_{\{t_1,\ldots,t_J\}}$,  consisting of all linear combinations of  $ S(t_1,\cdot), S(t_2,\cdot), \ldots, S(t_J,\cdot) $.
Let ${\cal{H}}^*$ be the union of these ${\cal{H}}_{\{t_1,\ldots,t_J\}}$'s over all $J$ and all values of $t_1, \ldots, t_J$.
Let $<,>$ be the inner product on ${\cal{H}}^*$ generated by  $< S(t_j,\cdot), S(t_k,\cdot)> = S(t_j,t_k)$, that is,
$< \sum_j a_j S(t_j,\cdot), \sum_k b_k S(x_k, \cdot)> = \sum_{j,k} a_j b_k   S(t_j, x_k)$.
Let ${\cal{H}}_S$ be the completion of ${\cal{H}}^*$ under the norm associated with this inner product.   
Then ${\cal{H}}_S$ is a Reproducing Kernel Hilbert Space with reproducing kernel $S$.  
So, by Theorem \ref{thm:finite},
 the solution to (\ref{eq:Gaussian.Bayes}) is of the form $\mu(t) = \sum_{l=1}^n \beta_l S(t_l,t)
 = S(t, \bf{t}) \bbeta $, with the $\beta_j$'s 
chosen to minimize
\[
\sum_{j=1}^n \left[Y_j -  \sum_{l=1}^n \beta_l S(t_l,t_j)\right]^2 + {\sigma}^2  
\sum_{l,k=1}^n  \beta_l \beta_k S(t_l,t_k)
= || {\bf{Y}} - S(\bf{t}, \bf{t}) \bbeta ||^2 +  {\sigma}^2  \bbeta' S(\bf{t}, \bf{t}) \bbeta
\]
where $\bbeta= (\beta_1,\ldots, \beta_n)'$.
The minimizing 
$
{\boldsymbol{\hat{\beta}}}$ is $ \left [  {\sigma^2}  {\rm{I}} + S(\bf{t}, \bf{t}) \right]^{-1}   {\bf{Y}},
$
and so  the solution to (\ref{eq:Gaussian.Bayes}) is equal to $\hat{\mu}_B$.

%%%%%%%%%%%%%%%%%%%%%%%%%%%%%%%%%%%%%%%%%%%%%%%%%%%%

\section{Results for the cubic smoothing spline}
\label{sec:cubic}

Here, we minimize (\ref{eq:splines}) using Theorem \ref{thm:finite}.  The expressions
for the reproducing kernels
 $R_0$ and $R_1$ are provided.  The next section contains an algorithm for computing $R_0$ and
 $R_1$ for general L.

The first step to minimize (\ref{eq:splines}) over $\mu \in {\cal{H}}^2[a, b]$
 is to define the inner product on ${\cal{H}}^2[a, b]$:
\[
<f, g> = f (a)g(a) + f'(a)g'(a) + \int_a^b f''(t)~g''(t)~dt.
\]
Verifying that this is an inner product is straightforward, including showing that $<f,f>=0$ if and only if
$f \equiv 0$.  The proof that ${\cal{H}}^2[a, b]$ is complete under this inner product 
uses the completeness of ${\cal{L}}^2[a,b]$.

For (C.1) and (C.2) of
  Section \ref{sec:general}, we partition ${\cal{H}}^2[a, b]$ into ${\cal{H}}_0$ and ${\cal{H}}_1$:
\[
{\cal{H}}_0 = \{ f: f''(t) \equiv 0 \} = {\rm{~the~span~of~}} \{ 1, t\}
\]
and
\[
{\cal{H}}_1 = \{ f \in {\cal{H}}^2[a,b]: f(a)=f'(a) = 0 \}.
\]
${\cal{H}}_1$
is  the orthogonal complement of ${\cal{H}}_0$ and so
${\cal{H}}^2[a,b] = {\cal{H}}_0 \oplus  {\cal{H}}_1$.  (This is shown in Theorem \ref{thm:Hilbert} for ${\cal{H}}^m[a,b]$.)

For (C.3) let 
\[
R_0(s,t) = 1 + (s  -   a)(t  -   a)
\]   
and
\[
R_1(s,t) = st \left( \min\{s, t\} -  a\right)
 + \frac{ s + t}{2} \left[   (\min\{s, t\})^2 -  a^2\right]
 + \frac{1}{3}  \left[  (\min\{s, t\})^3 -  a^3\right].
 \]
Then direct calculations verify that $R_0$ and $R_1$ are the reproducing kernels of, respectively, ${\cal{H}}_0$
and ${\cal{H}}_1$, that is, that $R_i \in {\cal{H}}_i$ and that 
$<R_i(\cdot,t),f > = f (t) $  for all $ f \in {\cal{H}}_i$, $i=0,1$. 

To verify that condition (C.4) is
satisfied, write $\mu = \mu_0 + \mu_1$, with $\mu_i \in {\cal{H}}_i$, $i=0,1$. Then
$P(\mu) = \int (\mu'')^2 = \int (\mu_1'')^2 = < \mu_1,\mu_1>$.

We can show that $F_j(\mu) = \mu(t_j)$ is a continuous linear functional, either by using the definition of the inner product
to verify continuity of $F_j$ or by noting that $R=R_0 + R_1$ is the reproducing kernel of $H^2[a,b]$. 
Thus, by Theorem \ref{thm:finite}, to minimize (\ref{eq:splines}) we can restrict attention to  
\[
\mu(t) = \alpha_0 + \alpha_1t +
\sum_1^n \beta_j   R_1(t_j, t)
\]
and find  $\alpha_0$, $\alpha_1$ and {\mbox{\boldmath{$\beta$}}}$\equiv (\beta_1,\ldots, \beta_n)'$ to minimize
\[
\sum_j [ Y_j - \alpha_0 - \alpha_1t_j - \sum_k \beta_k   R_1 (t_j,t_k) ]^2
+  \beta' {K} \beta
\]
where
${K}[j,k] = R_1(t_j,t_k)
$.  In matrix/vector form, we seek $\bbeta$ and ${\balpha}= (\alpha_0, \alpha_1)'$
to minimize
\begin{equation}
||{\bf{Y}}  -  T {\balpha}-  {{ K}} {\bbeta}||^2 +  \lambda \bbeta' K \bbeta
  \label{eq:splines.matrix}
  \end{equation}
with  
${\bf{Y}} = (Y_1,\cdots, Y_n)', T_{i1} = 1$ and $T_{i2} = t_i$,
$i = 1, \cdots, n$.  One can minimize (\ref{eq:splines.matrix}) directly, using matrix
calculus.

Unfortunately, solving the matrix equations resulting from the differentiation of (\ref{eq:splines.matrix}) involves inverting matrices which are  ill-conditioned and large. Thus,
the calculations are subject to round-off errors that seriously effect the accuracy of the solution. In addition, the matrices to be inverted are not sparse,
so that $O(n^3)$ operations are required. This can be a formidable task for, say,
$n = 1000$. The problem is due to the fact that the bases functions 1, $t$, and
$R_1(t_j, \cdot)$ are almost dependent with supports equal to the entire interval
$[a, b]$.
There are two ways around this problem. One way is to replace this inconvenient basis with a more stable one, one in which the elements have close
to non-overlapping support. The most popular stable basis for this problem is that
made up of cubic  B-splines (see, e.g., Eubank \cite{Eubank}). The $i$th B-spline basis function has support $[t_i; t_{i+2}]$ and thus the matrices involved in the minimization
of (\ref{eq:splines}) are banded, well-conditioned, and fast to invert. Another approach
is that of Reinsch (\cite{Reinsch1967},  \cite{Reinsch1970}). The Reinsch algorithm yields a minimizer
in O($n$) calculations. The approach for the Reinsch algorithm is based on a
paper of Anselone and Laurent \cite{Anselone.and.Laurent}. 
Section \ref{sec:minsquares2} gives this technique for minimization of expressions like
(\ref{eq:splines.matrix}).

%%%%%%%%%%%%%%%%%%%%%%%%%%%%%%%%%%%%%%%%%

\section{Results for penalties with differential operators}
\label{sec:general_differential}

Now consider the problem of minimizing (\ref{eq:general}) with penalty $P$ based on a differential operator  L,
as in (\ref{eq:differentialP}), that is, of minimizing
\begin{equation}
\label{eq:minimum_differentialP}
G( t_1,\ldots, t_n, Y_1,\ldots, Y_n, F_1(\mu),\ldots, F_n(\mu)) + \lambda \int ({\rm{L}}\mu)^2
\end{equation}
over $\mu \in {\cal{H}}^m[a,b]$. 
We can apply Theorem \ref{thm:finite}  using the Reproducing Kernel Hilbert Space
structure for ${\cal{H}}^m[a,b]$ defined in Section \ref{sec:Pen_theory} below.
We can then explicitly calculate 
the form of $\mu$ provided we can calculate reproducing kernels.  Theorem \ref{thm:Hilbert}  states a method for 
explicitly calculating reproducing kernels.
Section \ref{sec:algorithm} summarizes the algorithm for calculating reproducing kernels and the form 
of the minimizing $\mu$, and contains three examples of calculations.
Theorem \ref{thm:Hilbert}  and the calculations of Section \ref{sec:algorithm} require results from the theory of differential equations. 
 The Appendix contains these results,  including a constructive proof of the existence of $G( \cdot,\cdot )$, the Green's function
 associated with the differential operator L.  Section \ref{sec:minsquares2}  contains a fast algorithm for minimizing
 (\ref{eq:minimum_differentialP}) when $G$ is a sum of squares and $F_j(f) = f(t_j)$.

\subsection{The form of the minimizer of (\ref{eq:minimum_differentialP})}
\label{sec:Pen_theory}

Giving the form of the minimizing $\mu$ uses the result of
Theorem \ref{thm:A:Lnullbasis} in the Appendix,
 that there exist linearly independent $u_1,\cdots, u_m \in {\cal{H}}^m[a,b]$ with $m$
derivatives and that these functions   form a basis for the set
of all $\mu$ with L$\mu(t) = 0$ almost everywhere $t$.   Furthermore ${W}(t)$, the
Wronskian matrix associated with $u_1,\cdots, u_m$, is invertible for all $t \in
[a, b]$. The Wronskian matrix is defined as 
\[
[{W}(t)]_{ij} = u_i^{(j -  1)} (t), i, j =
1, \cdots, m.
\]

The following is an inner product  under which $ {\cal{H}}^m[a,b]$ is a Reproducing
Kernel Hilbert Space:
\begin{equation}
\label{eq:inner}
<f, g> = \sum_{j=0}^{m-1}f^{(j)} (a)g^{(j)}(a)
+ \int_a^b(\rmL f)(t)~(\rmL g)(t)~dt.
 \end{equation}
To show that this is, indeed, an inner product is straightforward, except to show that $<f,f> = 0$ implies that $f\equiv 0$. But this follows immediately
from Theorem \ref{thm:A:inner} in the Appendix.

\vskip 10pt

\noindent
\begin{theorem}
\label{thm:Hilbert}
 Let  L be as in  (\ref{eq:differentialP}), let  $\{u_1, \cdots, u_m\}$
 be a basis for the set of $\mu$ with L$\mu \equiv 0$ and let
${W}(t)$ be the associated Wronskian matrix. Then, under the inner product (\ref{eq:inner}),
${\cal{H}}^m[a,b]$ is a Reproducing Kernel Hilbert Space with reproducing kernel $R(s, t) =
R_0(s, t) + R_1(s, t)$ where
\[
R_0(s, t) = \sum_{i,j=1}^m C_{ij} u_i(s)u_j(t)
 \]
 with
 \[
  C_{ij} = \left[({W}(a){W}'(a)) ^{-1}\right]_{ij},
\]
\[ R_1(s,t) = \int_{u=a}^b G(s,u)~ G(t,u)~du
\]
and $G(\cdot , \cdot ) $ is the Green's function associated with L, as given in equations (\ref{eq:Greensint}),
(\ref{eq:boundary}) and (\ref{eq:Greens}) in the Appendix.
Furthermore, ${\cal{H}}^m[a,b]$ can be partitioned into the direct sum of the two subspaces
\begin{eqnarray}
{\cal{H}}_0 & =& ~\text{the set of all}~ f \in {\cal{H}}^m[a,b]
~\text{with}~ \rmL f (t) = 0 ~\text{almost everywhere}~ t
\nonumber \\
&=&
~\text{the span of } u_1,\ldots,u_m
\nonumber \end{eqnarray}
and
\[
{\cal{H}}_1 =  ~\text{the set of all}~ f \in {\cal{H}}^m[a,b]
~\text{with}~ f^{(j)}(a) = 0, j = 0, \cdots  m-1.
\]
${\cal{H}}_1$  is the orthogonal complement of  ${\cal{H}}_0$.
${\cal{H}}_0$ has reproducing kernel $R_0$ and ${\cal{H}}_1$ has reproducing kernel $R_1$.
\end{theorem}

\begin{proof} To prove the Theorem, it suffices to show the following.
\begin{itemize}
\item[(a)] Any $f$ in ${\cal{H}}^m[a,b]$ can be written as $f = f_0 + f_1$, with $f_i \in {\cal{H}}_i$  and
$<f_0,f_1> = 0$.
\item[(b)] $R_0$ is the reproducing kernel for ${\cal{H}}_0$
and $R_1$ is the reproducing kernel for ${\cal{H}}_1$.
\end{itemize}

Consider (a).
Obviously,  for $f_i \in {\cal{H}}_i$, $i=0,1$,  $<f_0, f_1>$ is
equal to zero, by the definition of the inner product in (\ref{eq:inner}).
To complete the proof of (a), fix $f \in {\cal{H}}^m[a,b]$
and  find $c_1, \cdots, c_m$ such that, if $f_0 = \sum c_i u_i$, then $f_1 =
f - f_0 \in {\cal{H}}_1$. That is, we find $c_1,\ldots, c_m$ such that, for $j = 0, \cdots,  m - 1, f ^{(j)}_1 (a) = 0$, that
is $f^{(j)}(a)  -  \sum_i c_i u^{(j)}_i (a) = 0$. Writing this in matrix notation and using the
Wronskian matrix yields
\[
(f (a), f'(a),\cdots, f^{(m - 1)}(a)) = (c_1,\cdots, c_m) {W}(a)
\]
and we can solve this for $(c_1,\cdots, c_m)$, since the Wronskian ${W}(a)$ is invertible.

Consider (b).
To prove that $R_1$ is the reproducing kernel for ${\cal{H}}_1$, first simplify notation,
fixing $t \in [a, b]$ and letting $r(s) = R_1(s, t)$.  
We must show that $r \in {\cal{H}}_1$ and that 
that $<r, f> = f (t)$ for all $f \in {\cal{H}}_1$.  Again, to simplify notation,
let  $h(u) = G(t, u)$.  By definition of $R_1$,  $r(s) = \int_a^b G(s, u)~h(u)~du$.  
By Theorems \ref{thm:A:Greens_functiona} and  \ref{thm:A:Greens_functionb},  $r \in {\cal{H}}_1$ and
L$r(s) = h(s) = G(t, s)$ almost everywhere $s$. Therefore, for $f \in {\cal{H}}_1$ ,
\[
<r, f> =  0 + \int_a^b ({\rm L} r)(s)~({\rm L} f )(s)~ ds
= \int_a^b G(t, s)~({\rm L} f )(s)~ds = f (t)
\]
by the definition of the Green's function.   See equation (\ref{eq:Greensint}).

Now consider $R_0$.  Obviously, $R_0(\cdot , t) \in {\cal{H}}_0$,
since it is a linear combination of the $u_i$'s. To show that $<R_0(\cdot , t), f> = f (t)$,
it suffices to consider $f = u_l, l = 1,\cdots, m$. Noting that L$u_l \equiv 0$, write
\begin{eqnarray*}
<R_0(\cdot , t), u_l>
&=&\sum_{i,j=1}^m C_{ij}~ u_j(t)~<u_i,u_l>  \\
&=&\sum_{i,j=1}^m C_{ij}~ u_j(t) \left[ \sum_{k=0}^{m-1} u_i^{(k)}(a) u_l^{(k)}(a)  ~~~ + 0 ~ \right]\\
&=&\sum_{i,j=1}^m C_{ij}~ u_j(t) \sum_{k=0}^{m-1} [{W}(a)]_{i,k+1}[{W}(a)]_{l,k+1}\\
&=&\sum_{i,j=1}^m C_{ij}~ u_j(t)[{{W}(a){W}'(a)}]_{li}\\
&=&\sum_{j=1}^m u_j(t)[{{W}(a){W}'(a)}{\bf{C}}]_{lj}\\
&=& u_l(t).
\end{eqnarray*}
\end{proof}

\vskip 10pt
We can now use Theorems \ref{thm:finite} and  \ref{thm:Hilbert} to write the form of the minimizer of (\ref{eq:minimum_differentialP}).  The proof
of the following Theorem is straightforward.

\vskip 10pt \noindent
\begin{theorem}
\label{thm:Hilbert2}
 Suppose that L is as in (\ref{eq:differentialP}). Let $u_1,\cdots, u_m$ be a basis for the
set of $\mu$'s with L$\mu \equiv 0$ and let $G$ be the corresponding Green's function,
defined in equations (\ref{eq:Greensint}),
(\ref{eq:boundary}) and (\ref{eq:Greens}) in the Appendix.  Let 
\[
R_1(s, t) = \int_a^b G(s, u)~ G(t, u)~ dt
\]
and
$\eta_{j1}(t) = F_j(R_1(\cdot,t))$.
Then the
minimizer of (\ref{eq:minimum_differentialP}) must be of the form
\[
\mu(t) = \sum_{j=1}^m \alpha_j u_j(t)
+\sum_{j=1}^n \beta_j  \eta_{j1}(t)
\]
where the $\alpha_j$'s and $\bbeta \equiv (\beta_1,\ldots,\beta_n)'$  minimize
\[
G( t_1,\ldots, t_n, Y_1,\ldots, Y_n, F_1(\mu),\ldots, F_n(\mu)) + \lambda \bbeta' {K} \bbeta
\]
with ${K}$ as defined in Theorem \ref{thm:finite}.
\end{theorem}

%%%%%%%%%%%%%%%%%%%%%%%%%%%%%%%%%%%%%%%%%%%%%%%%%%%%%%

\subsection{Algorithm and examples for calculating $R_0$, $R_1$ and the minimizing  ${\mu}$}
\label{sec:algorithm}

Suppose that we're given a linear differential operator L as in (\ref{eq:differentialP}).
The following steps summarize results so far, describing how to calculate $R_0$ and $R_1$, the required
reproducing kernels associated with L, and the $\mu$ that minimizes (\ref{eq:minimum_differentialP}).  

\begin{enumerate}
\item Find $u_1,\cdots, u_m$, a basis for the set of functions $\mu$ with L$\mu \equiv 0$. 
\item 
Calculate ${W}(\cdot )$, the Wronskian of the $u_i$'s: ${W}_{ij} (t) = u^{(j-1)}_i (t)$.
\item  Set 
 $R_0(s, t) = \sum_{i,j} [ [{W}(a){W}'(a)]^{-1}]_{ij} u_i(s)u_j(t)$.
\item
Calculate $(u^{*}_1(t),\cdots, u^{*}_m(t))$, the last row of the inverse of ${W}(t)$.
\item
 Find $G$, the associated Green's function: $G(t, u) = \sum u_i(t) u^{*}_i (u)$ for
$u \leq t$, 0 else.
\item
Set $R_1(s, t) = \int_a^b G(s, u)~G(t, u)~du$.
\item Find $\eta_{1j}$:  $\eta_{1j}(t) = {{F}}_j(R_1(\cdot,t))$.
\item Calculate the symmetric matrix ${K}$:  
 ${K}[j,k]= F_k(\eta_{1j})$.  If ${{F}}_j(\mu) = \mu(t_j)$ 
 and ${{F}}_k(\mu) = \mu(t_k)$ then ${K}[j,k] = R_1(t_j,t_k).$
\item Set $\mu(t) = \sum \alpha_j u_j(t) + \sum_j \beta_j \eta_{1j}(t)$ and minimize 
$G( t_1,\ldots, t_n, Y_1,\ldots, Y_n,$ $ F_1(\mu),\ldots, F_n(\mu)) + \lambda \bbeta' {K} \bbeta$ with
respect to $\bbeta$ and the $\alpha_j$'s.
\end{enumerate}

The first step is the most challenging, and for some L's, it may in fact be impossible to find the $u_j$'s
in closed form.   However, if L is
a linear differential operator with constant coefficients, then the first step is easy, using Theorem \ref{thm:A:solution}.
Alternatively, if one has an approximate model in mind defined in terms of  
known functions $u_1,\ldots, u_m$,  then one can find the corresponding L (see Example
3 below).

The reader can use these steps to derive the expressions in Section
\ref{sec:cubic} for the cubic smoothing spline.

Although the calculation of the minimizing $\mu$ does not involve $R_0$, step 3 is included for completeness,
to allow the reader to calculate the reproducing kernel, $R_0 + R_1$, for ${\cal{H}}^m[a,b]$ under the inner product
(\ref{eq:inner}).

%%%%%%%%%%%%%%%%%%%%%%%%%%%%%%%%%%%%%%%%%%%%%%%%%

\vskip 10pt \noindent
{\bf{{\em{Example 1}}}}.
Suppose that  L$\mu = \mu'$ and that the interval $[a,b]$ is equal to $[0,1]$. 
In Step 1, the basis for L$\mu \equiv 0$ is $u_1(t) = 1$.   In Step 2, the Wronskian is the one by one
matrix with element equal to 1. So in Step 3, $R_0(s,t) \equiv 1$.   In Step 4,
$u^{*}_1(s) = 1$ and so, in Step 5, $G(t, u) = 1$ if $u \leq t$, 0 else. Therefore
\[
R_1(s, t) = \int_0^{\min\{s,t\}} 1~ du = \min\{s,t\}.
\]
Thus, we seek $\mu$ of the form
\[
\mu(t) = \alpha + \sum_{j=1}^n \beta_j {{F}}_j(R_1(\cdot,t)).
\]

If ${{F}}_j(\mu) = \mu(t_j)$, $j=1,\ldots,n$, then we seek
\[
\mu(t) =  \alpha + \sum_{j=1}^n \beta_j  \min \{ t_j,t\},
\]
that is, the minimizing $\mu$ is piecewise linear with pieces defined in terms of $t_1,\ldots, t_n$.
In Step 8, ${K}[j,k]= \min \{ t_j, t_k \}$.

If, instead,  $F_j(\mu) = \int_0^1  f_j \mu$ for known $f_j$, as in Section \ref{sec:integral_F}, then 
\begin{eqnarray*}
F_j(R_1(\cdot,t) )&=& \eta_{1j}(t) = \int^1_0 f_j(s)~R_1(s, t)~ ds 
= \int^1_0 f_j(s)~\min\{s,t\}~ ds 
\\
&=&
\int^t_0 s~f_j (s)~ ds + t \int^1_t f_j(s) ~ds
\end{eqnarray*}
and, in Step 8,
\[
{K}[j,k] = \int_{t=0}^1  f_k(t)~ \eta_{1j}(t)~dt = \int _{s,t=0}^1  f_k(t) ~ f_j(s) ~ \min\{s,t\} ~ds~dt.
\]

\vskip 10pt \noindent
{\bf{\em{Example 2}}}. Suppose that L$f = f'' + \gamma f'$, $\gamma$ a real number.

For Step 1, we can find $u_1$ and $u_2$ via Theorem \ref{thm:A:solution} in the Appendix. We first solve
$x^2 + \gamma x = 0$ for the two roots, $r_1 = 0$ and $r_2 =  -\gamma$. So
\[
u_1(t) = 1 ~\text{and}~ u_2(t) = \exp( -  \gamma t).
\]
For Step 2, we compute the Wronskian
\[
{W}(t) =
\left[ {\begin{array}{cc}
1&0 \\
\exp(-\gamma t)& -\gamma \exp(-\gamma t)\\
\end{array}}\right].
\]
For Step 3 we have
\[
[{W}(a){W}'(a)]^{-1} =
\left[ {\begin{array}{cc}
1+\frac{1}{\gamma^2} &-\frac{1}{\gamma^2}\exp(\gamma a)\\
-\frac{1}{\gamma^2}\exp(\gamma a)& \frac{1}{\gamma^2}\exp(2\gamma a)\\
\end{array}}\right].
\]
So
\begin{eqnarray*}
R_0(s, t)
&=&C_{11}u_1(s)u_1(t) + C_{12}u_1(s)u_2(t) + C_{21}u_2(s)u_1(t) + C_{22}u_2(s)u_2(t)\\
&=& 1 +\frac{1}{\gamma^2}  -   \frac{1}{\gamma^2} \exp( - \gamma t^{*} )  -   \frac{1}{\gamma^2} \exp( -  \gamma s^{*} ) + \frac{1}{\gamma^2} \exp (-\gamma (s^{*}+t^{*})).
\end{eqnarray*}
with $s^{*}= s-a$ and $t^{*}= t-a$.

For Step 4, inverting ${W}(t)$ we find that
\[
u^{*}_1(t) = \frac{1}{\gamma}~ \text{and} ~ u^{*}_ 2(t) = -\frac{1}{\gamma} \exp(\gamma t)
\]
and so, in Step 5, the Green's function is given by
\begin{eqnarray*}
G(t, u) & = & \begin{cases}
\frac{1}{\gamma} \left( 1-\exp(-\gamma(t-u)) \right) & \mbox{for $u \leq t$}\\
0 & \mbox{else}.\end{cases}
\end{eqnarray*}

To find $R_1(s, t)$ in Step 6, first suppose that $s \leq t$. Then

\begin{eqnarray}
\label{eq:R1.1}
R_1(s, t)
&=&\int_a^s \gamma^{-2}(1-e^{-\gamma(s-u)})~(1-e^{-\gamma(t-u)})~du \nonumber \\
&=&-\frac{1}{\gamma^3} +\frac{s^{*}}{\gamma^2}+  \frac{1}{\gamma^3} \exp(-\gamma s^{*})
   +\frac{1}{\gamma^3}\exp(-\gamma t^{*}) \nonumber \\
& &-\frac{1}{2\gamma^3}\exp[-\gamma (t^{*}-s^{*})] - \frac{1}{2\gamma^3}\exp[-\gamma (s^{*}+t^{*})]. 
\end{eqnarray}
Since $R_1(s, t) = R_1(t, s)$, if $t < s$, then $R_1(s, t)$ is gotten by interchanging $s^{*}$
and $t^{*}$  in the above.  

Therefore, to minimize (\ref{eq:minimum_differentialP}) over $\mu \in {\cal{H}}^4[a,b]$, we seek 
$\mu$ of the form
\[
\mu(t) = \alpha_1  + \alpha_2 \exp( -  \gamma t) + \sum_1^n \beta_j F_j(R_1(\cdot,t).
\]
The calculations in Steps 7 and 8 for  $\eta_{j1}(t) = F_j(R_1(\cdot,t))$ and ${K}$ are tedious except in the case that
$F_j(f) = f(t_j)$.

\vskip 10pt\noindent
{\bf{\em{Example 3.}}}  Instead of specifying the operator L, one might more easily specify 
basis functions $u_1,\cdots, u_m$ for a preferred approximate parametric model.  For instance, one might think that
$\mu$ is approximately a constant plus a damped sinusoid:  $\mu(t) \approx \alpha_1 + \alpha_2 \sin (t)  \exp(-t)$.
Given $u_1,\cdots, u_m$, 
one can easily find the operator L so that L$u_i \equiv 0$, $i = 1,\cdots, m$, and thus one can
define an estimate of $\mu$ as the minimizer of (\ref{eq:minimum_differentialP}).   Assume 
that each $u_i$ has $m$ continuous derivatives and that
the associated Wronskian matrix ${W}(t)$ is invertible for all $t \in [a, b]$. 
To find L, we solve for the $\omega_j$'s in (\ref{eq:differentialP}):
\[
0 = (\rmL u_i)(t) = u^{(m)}_i (t) +
\sum^{m-1}_{j=0} \omega_j(t)u^{(j)}_i (t),
\]
that is
\[
u^{(m)}_i (t) = -\sum ^{m -1}_{j=0} \omega_j (t)u^{(j)}_i (t).
\]
This can be written in matrix/vector form as
\[
{W}(t)
\begin{bmatrix} 
\omega_0(t)\\
\vdots\\
\omega_{m-1}(t) \\
\end{bmatrix}
=
-  \begin{bmatrix}
u_1^{(m)}(t)\\
\vdots\\
u_m^{(m)}(t)\\
\end{bmatrix}
\]
yielding 
\[
\begin{bmatrix}
\omega_0(t)\\
\vdots\\
\omega_{m-1}(t) \\
\end{bmatrix}
= -{W}(t)^{-1}
\begin{bmatrix}
u_1^{(m)}(t)\\
\vdots\\
u_m^{(m)}(t)\\
\end{bmatrix}.
\]
Obviously, the $\omega_j$'s are continuous, by our assumptions concerning the $u_i$'s and the invertibility of
$W(t)$.

For the example with $u_1\equiv 1$ and $u_2 =\sin(t) \exp(-t)$, we find that 
\[
W(t) = \left[ \begin{matrix} 1  &  0  \\
  \sin(t) \exp(-t)  &  \exp(- t)  [ \cos(t) - \sin(t) ]
  \end{matrix}  \right],
  \]
which is invertible on $[a,b]$ provided $\cos(t) \neq \sin(t)$ for $t \in [a,b]$.    In this case,
$\omega_0(t) \equiv 0$,
$\omega_1(t) = 2 \cos (t)/[ \cos(t) -  \sin(t) ] $ and
 so the associated differential
operator is L$(\mu)(t) = \mu''(t)  +  2  \mu'(t) \cos (t)/[ \cos(t) -  \sin(t) ].$  Note that we do not need L to proceed with the minimization
of (\ref{eq:minimum_differentialP}) -- we only need $u_1,\cdots,u_m$ to calculate the required reproducing kernels.
However, if we would like to cast the problem in the Bayesian model of Section \ref{sec:Bayes}, we require L.

%%%%%%%%%%%%%%%%%%%%%%%%%%%%%%%%%%%%%%%%%%%%%%%%%

\subsection{Minimization of the penalized weighted sum of squares via matrix calculus}
\label{sec:minsquares1}

 Consider minimizing a  specific form of (\ref{eq:minimum_differentialP}) over $\mu \in {\cal{H}}^m[a,b]$, namely minimizing
 \begin{equation}
 \label{eq:G_sum_of_square}
  \sum_j d_j [Y_j -F_j(\mu)]^2  +  \lambda \int ( \rmL u )^2
 \end{equation}
for known and positive $d_j$'s.    We can rewrite this as a minimization problem
easily solved by matrix/vector calculations, provided we can find a basis $\{u_1,\ldots, u_m\}$ for the
set of $\mu$ with L$\mu = 0$.  

Theorem \ref{thm:Hilbert2} implies that,  to minimize (\ref{eq:G_sum_of_square}), we 
must find
$\boldsymbol{\alpha}=(\alpha_1,\ldots,\alpha_m)'$ and $\boldsymbol{\beta}= (\beta_1,\ldots,\beta_n)'$ to minimize
\[
%\label{eq:matrixform}
({\bf{Y}}  -   T \boldsymbol{\alpha} -   {K} \boldsymbol{\beta})' {{D}}
({\bf{Y}}  -   T \boldsymbol{\alpha}  -   {K} \boldsymbol{\beta})
+ \lambda \boldsymbol{\beta}' {K} \boldsymbol{\beta}
\]
where ${\bf{Y}}  = (Y_1,\cdots, Y_n)'$, $T $ is $n \times m$ with  $T[i,j] = u_j(t_i)$,
${K}$ is $n \times n$ with ${K}[j,k] = {F}_j(\eta_{k1})$,
and ${{D}}$ is an $n$ by $n$ diagonal matrix with ${{D}}[i,i] = d_i$. 
Assume, as is typically the case, that $T$ is of full rank and $K$ is invertible.   
Taking  derivatives with respect to $\boldsymbol{\alpha}$ and $\bbeta$ and setting equal to zero yields
\begin{equation}
\label{eq:matrixform2}
{{T'D}}({\bf{Y}}-{K} \boldsymbol{\hat{\beta}})
={{T'D}}T \hat{\boldsymbol{\alpha}}.
\end{equation}
and\[
 -2{K'D}
({\bf{Y}}  -   T \hat{\boldsymbol{\alpha}}  -   {K} \boldsymbol{\hat{\beta}})
+ 2\lambda  {K} \boldsymbol{\hat{\beta}}=0
\]
which is equivalent
to
\[
{\bf{Y}}  -   T \hat{\boldsymbol{\alpha}}
- ({K}+\lambda { D}^{-1}) \boldsymbol{\hat{\beta}} = 0.
\]
Let
\[
{{M}} = {K}+\lambda { D}^{-1}.
\]
Then
\begin{equation}
  \boldsymbol{\hat{\beta}} = {{M}}^{-1} ({\bf{Y}}  -   T \hat{\boldsymbol{\alpha}} ).
\label{eq:betahat}
\end{equation}
Substituting this into (\ref{eq:matrixform2}) yields
\[
{{T'D}}[I- {{KM}}^{-1}]{\bf{Y}}  =
{{T'D}}[I- {{KM}}^{-1}]  T\hat{\boldsymbol{\alpha}},
\]
that is
\[
{{T'D}}[{{M}}- {K}] {{M}}^{-1} {\bf{Y}}  =
{{T'D}}[{{M}}- {K}] {{M}}^{-1} T\hat{\boldsymbol{\alpha}}
\]
or
$
\lambda {{T'}}   {{M}}^{-1} {\bf{Y}}  =
\lambda {{T'}}  {{M}}^{-1} T\hat{\boldsymbol{\alpha}}.
$

Therefore, provided $T$ is of full rank,
\begin{eqnarray}
\label{eq:alpha}
\hat{\boldsymbol{\alpha}} = ({{T'M}}^{-1}  T)^{-1}{{T'M}}^{-1}  {\bf{Y}}
\end{eqnarray}
and
\begin{eqnarray}
\label{eq:beta}
\boldsymbol{\hat{\beta}} = {{M}}^{-1}
[{\rm{I}}-T ({{T'M}}^{-1}  T)^{-1} {{T'M}}^{-1}]
{\bf{Y}}.
\end{eqnarray}

Unfortunately, using equations (\ref{eq:alpha}) and (\ref{eq:beta}) results in computational problems
since typically ${{M}}$ is an ill-conditioned matrix and thus difficult to invert.  Furthermore,
${{M}}$ is $n \times n$ and $n$ is typically large, making inversion expensive.  Fortunately,
when $F_j (f) = f(t_j)$ we can transform the problem to alleviate the difficulties and to speed computation.
The details are given in the next section.

%%%%%%%%%%%%%%%%%%%%%%%%%%%%%%%%%%%%%%%%%%%%%%%%

\subsection{Algorithm for minimizing  the penalized weighted sum of
squares when $F_j (f) = f(t_j)$}
\label{sec:minsquares2}

Assume that $F_j (f) = f(t_j)$, that $a < t_1 < \cdots < t_n < b$, that $T$ is of full rank $n-m$ and that $K$ is invertible.
The goal is to re-write $\hat{\boldsymbol{{\alpha}}}$ in (\ref{eq:alpha}) and  $ \boldsymbol{\hat{\beta}} $  in (\ref{eq:beta})
so that we only need to invert small or banded matrices.  
Meeting this goal  involves defining a ``good" matrix  $Q$ and showing that
\begin{equation}
\label{eq:beta2}
\boldsymbol{\hat{\beta}} = Q({{Q'MQ}})^{-1} Q'{\bf{Y}}
\end{equation}
and
\begin{equation}
\label{eq:alpha2}
\hat{\boldsymbol{{\alpha}}} =  (T'T)^{-1} T' ({\bf{Y}} -   {{M}} {\bf{\hat{\beta}}}) .
\end{equation}

We will define $Q$ so that ${{Q'MQ}}$ is banded and thus easy to invert.
To begin, let $Q$ be an $n$ by $n-m$ matrix of full column rank such that ${{Q'T}}$ is
an $n-m$ by $m$ matrix of zeroes. $Q$ isn't unique, but later, further restrictions will be placed
on  $Q$ so that ${{Q'MQ}}$ is banded.  

We first show that ${{T'}}\boldsymbol{\hat{\beta}} = 0$. This will imply that there exists an 
$n-m$ vector $\boldsymbol{\gamma}$ such that $\boldsymbol{\hat{\beta}} = Q \boldsymbol{\gamma}$. 
From (\ref{eq:betahat}) 
\begin{equation}
\label{eq:Yhat}
{\bf{Y}} = {{M}} \boldsymbol{\hat{\beta}} + T \hat{\boldsymbol{\alpha}}
\end{equation}
Substituting this into (\ref{eq:alpha}) yields
\[
\hat{\boldsymbol{\alpha}} = ({{T'M}}^{- 1}T)^{-1}T' \boldsymbol{\hat{\beta}} + \hat{\boldsymbol{\alpha}}.
\]
Therefore
\[
({{T'M}}^{- 1}T)^{-1}T' \boldsymbol{\hat{\beta}} = 0
\]
and so $T' \boldsymbol{\hat{\beta}} = 0$ and $\boldsymbol{\hat{\beta}} = Q \boldsymbol{\gamma}$ for some $\boldsymbol{\gamma}$. To find $\boldsymbol{\gamma}$, use (\ref{eq:betahat}):
\[
{{Q'M}} \boldsymbol{\hat{\beta}} = Q'({\bf{Y}}  -   T \hat{\boldsymbol{\alpha}}) = Q'{\bf{Y}}
\]
since ${{Q'T}} = 0$. So ${{Q'MQ}}\boldsymbol{\gamma}= Q'{\bf{Y}}$, yielding
\[
\boldsymbol{\gamma} = ({{Q'MQ}})^{-1} Q'{\bf{Y}}.
\]
Therefore equation (\ref{eq:beta2}) holds.  Equation (\ref{eq:alpha2}) follows immediately
from equation (\ref{eq:Yhat}).

We can also
find an easy-to-compute form   for 
$\hat{\bf{Y}} \equiv T \hat{\boldsymbol{\alpha}} + {K} \boldsymbol{\hat{\beta}}$ using (\ref{eq:Yhat}):
\[
{\bf{Y}} = ({K}  + \lambda {{D}} ^{- 1}) \boldsymbol{\hat{\beta}} + 
T \hat{\boldsymbol{\alpha}} = \hat{\bf{Y}} + \lambda {{D}}^{-1} \boldsymbol{\hat{\beta}}
\]
and so 
\[
\hat{\bf{Y}} = {\bf{Y}}  -   \lambda {{D}}^{-1} {\boldsymbol{\hat\beta}}.
\]

Note that we have not yet used the fact that $F_j(f) = f(t_j)$.
In the special case that $F_j(f) = f(t_j)$, we can choose $Q$ so that ${{Q'MQ}}$ is banded. 
Specifically, in addition to requiring that ${{Q'T}} = 0$, we also seek $Q$ with
\begin{equation}
\label{eq:Qnice}
Q_{ij} = 0 ~\text{unless}~ i = j, j + 1,\cdots, j + m.
\end{equation}
So we want $Q$ with $[{{Q'T}}]_{ij} = \sum_{l=0}^m Q_{i+l,i}u_j(t_{i+l}) = 0$ for all $j = 1,\cdots, m,
i = 1,\cdots, n -m$. That is, for each $i$, we seek an $(m + 1)$-vector ${\bf{q}}_i \equiv
(Q_{ii},\cdots, Q_{i+m,i})'$ satisfying ${\bf{q}}_i'T_i = 0$, where $T_i$ is the $(m + 1)$ by $m $ matrix with
$lj$th entry equal to $u_j(t_{i+l})$. This is easily done by a QR decomposition of $T_i$:   the matrix
$T_i $ can be written as $T_i =  Q_i ~ R_i$ for some $Q_i $, an $(m+1) \times (m+1)$ orthonormal matrix,
 and some $ R_i, ~(m+1) \times m$ with last row equal to 0.  Take  $ {\bf{q}}_i $ to be the last column of $Q_i$.

We now show that ${{Q'MQ}}$ is banded, specifically, that $[{{Q'MQ}}]_{kl} = 0$
whenever $|k-l|>m$. Write ${{Q'MQ}} = {{Q'KQ}} + \lambda {{Q'D}}^{-1}Q$. Since
${{D}}$ is diagonal, one easily shows that $[{{QD}}^{- 1}Q]_{kl} = 0$ for $|k-l|>m$. To show that the same is
true for ${{Q'KQ}}$, write
\begin{eqnarray}
{K}[i,j] 
&=& R_1(t_i, t_j) \nonumber \\
&=& \int G(t_i, \omega)~G(t_j, \omega)~ d\omega \nonumber \\
&=& \sum_{r,s} u_r(t_i)u_s(t_j) \int^{\min\{t_i,t_j\}}_a u^{*}_r(\omega)~u^{*}_s(\omega)~ d\omega\nonumber \\
&\equiv& \sum_{r,s} u_r(t_i)u_s(t_j) ~ {\cal{F}}_{r,s}(\min\{t_i,t_j\}). 
\nonumber \\
&=& \sum_{r,s} T_{ir}  T_{js}  ~ {\cal{F}}_{r,s}(\min\{t_i,t_j\}). 
%\label{eq:K}
\nonumber
\end{eqnarray}
Since ${{Q'KQ}}$ is symmetric, it suffices to show that $[{{Q'KQ}}]_{kl} = 0$ for $k-l>m$. So fix $k$ and $l$ with
$k-l > m$ and write
\begin{eqnarray*}
[{{Q'KQ}}]_{kl} 
&=& \sum_{i,j=1}^n Q_{ik} K_{ij} Q_{jl}
= \sum_{i,j=0}^m Q_{k+i,k} K_{k+i, l+j} Q_{l+j,l}
\\
&=&
\sum_{i,j=0}^m \sum_{r,s=1}^m Q_{k+i,k} ~ {\cal{ F}}_{r,s}(~\min\{t_{k+i},t_{l+j}\})~ T_{k+i,r} T_{l+j, s} Q_{l+j,l}\\
&= &\sum_{j=0}^m \sum_{r,s=1}^m  ~ {\cal{F}}_{r,s}(t_{l+j}) ~ T_{l+j, s} Q_{l+j,l} 
     \sum_{i=0}^m Q_{k+i,k} T_{k+i, r}.
\end{eqnarray*}
The last equality follows since $k > l + m$ and $ 0 \le i,j \le m$ imply that  $k+i > l+j $ and so $t_{l+j} < t_{k+i}$.
We immediately have that $[{{Q'KQ}}]_{kl}=0$, since  $\sum_{i=0}^m Q_{k+i,k} T_{k+i, r} = [{{Q'T}}]_{kr}= 0$.

\vskip 15pt

Thus minimizing
 (\ref{eq:G_sum_of_square}) when $F_j (f) = f(t_j)$  is easily and quickly done through the following steps.
 
 \begin{enumerate}
 \item
  Follow steps 1 through 8 of Section \ref{sec:algorithm}
   to find $u_1,\cdots, u_m$, a basis for L$\mu = 0$,
the reproducing kernel $R_1$ and the matrix 
${K}$:   $ {K}[i,j] = R_1(t_i, t_j)$.
\item
 Calculate the matrix  $T$:   $T[i,j] = u_j(t_i)$.
\item
 Find $Q$ $n$ by $(n-m)$ of full column rank satisfying   equation (\ref{eq:Qnice}) and ${{Q'T}} = 0$.
One can find  $Q$ directly or by the method outlined below equation (\ref{eq:Qnice}).
\item Find
$\hat{\bbeta}$ and $\hat{\balpha}$ using equations (\ref{eq:beta2}) and (\ref{eq:alpha2}).
Speed  the matrix inversion by using the fact that $Q'{{M}} Q$ is banded. \end{enumerate}

%%%%%%%%%%%%%%%%%%%%%%%%%%%%%
\vskip 10pt 
\noindent
{\bf{\em{Example 2 continued from Section \ref{sec:algorithm}}}}.
Suppose that we want to minimize
\[
\sum_{j=1}^n d_j (Y_j  - u_(t_j))^2 
+ \lambda \int^1_0 (\mu''(t) + \gamma \mu'(t))^2 ~dt
\]
over $\mu \in {\cal{H}}^2[0, 1]$. For simplicity,
assume that $t_i = i/(n + 1)$.
Using the calculations from Section \ref{sec:algorithm},  we set
$T_{i1} = 1, T_{i2} = \exp( -  \gamma t_i)$, and ${K}[i,j] = R_1(t_i, t_j)$, with
$R_1$ as in (\ref{eq:R1.1}).

For Step 3, we find $Q$ directly:  we seek $Q$ $n$ by $(n -2)$ with $Q_{ij} = 0$ unless $i = j, j + 1, j + 2$ and
\[
0 = [{{Q'T}}]_{ij}
= Q_{ii}T_{ij} + Q_{i,i+1}T_{i+1,j} + Q_{i,i+2}T_{i+2,j}.
\]
Thus, for $j=1$,
\[
0 = Q_{ii} + Q_{i,i+1} + Q_{i,i+2}
\]
and, for $j=2$,
\[
0 = Q_{ii} \exp( - \gamma t_i) + Q_{i,i+1} \exp( - \gamma t_{i+1}) + Q_{i,i+2} \exp( - \gamma t_{i+2}).
\]
We take
\[
Q_{ii} = 1-\exp\left( -\frac{\gamma} {n + 1}\right) \quad
Q_{i,i+1} = - \exp \left(\frac{\gamma} {n + 1}\right) +   \exp \left(- \frac{\gamma} {n + 1}\right)
\]
and
\[
Q_{i,i+2} = \exp\left(\frac{\gamma} {n + 1}\right) -   1:
\]
Continuing with the fourth step to find ${\bf{\hat{\alpha}}}$ and ${\bf{\hat{\beta}}}$ is straightforward.

%%%%%%%%%%%%%%%%%%%%%%%%%%%%%%%%%%%%%%%%%%%%

\appendix

\section{}

The Appendix  contains background on the solution of linear differential equations  L$\mu = 0$ with L as
in (\ref{eq:differentialP}). Section  \ref{sec:A:Greensfunction} contains results about $G$, the Green's function
associated with L.

\subsection{Differential Equations}
\label{sec:A:diffeq}

 Details of results in this section can be found
in Coddington \cite{Coddington}. The main Theorem, stated without proof, follows.

\noindent
\begin{theorem}
\label{thm:A:Lnullbasis}
 Let L be as in (\ref{eq:differentialP}). Then there exists $u_1,\cdots, u_m$ a basis for the
the set of all $\mu$ with  L$\mu \equiv 0$, with each $u_i$ real-valued and having $m$ derivatives. Furthermore,
any such basis will have an invertible Wronskian matrix ${W}(t)$ for all $t \in [a,b]$.
The Wronskian matrix is defined as
\[
[{W}(t)]_{ij} = u^{(j -  1)}_i \quad i, j = 1,\cdots, m.
\]
\end{theorem}

The following Theorem, stated without proof, is useful for calculating the
basis functions in the case that the $\omega_j$'s are constants.

\begin{theorem}
\label{thm:A:solution}
 Suppose that L is as in (\ref{eq:differentialP}), with the $\omega_j$'s real numbers. Denote
the $s$ distinct roots of the polynomial $x^m + \sum^{m - 1}_{j=0} \omega_j x^j$ as $r_1,\cdots, r_s$. Let $m_i$
denote the multiplicity of root $r_i$ (so $m = \sum^s_1 m_i$). Then the following $m$
functions of t form a basis for the set of all $\mu$ with  L$\mu \equiv 0$:
\[
\exp(r_i t), t \exp(r_i t),\cdots, t^{m_i - 1} \exp(r_i t) \quad  i = 1,\cdots, s.
\]
\end{theorem}

The following result, stated without proof, is useful for checking that a
set of functions does form a basis for the set of all $\mu$ with  L$\mu \equiv 0$.

\begin{theorem}
Suppose that $u_1,\cdots, u_m$ have $m$ derivatives on $[a, b]$ and that
L$u_i \equiv 0$.  If ${W}(t_0)$ is invertible at some $t_0 \in [a, b]$, then the $u_i$'s are linearly
independent, and thus a basis for the set of all $\mu$ with  L$\mu \equiv 0$.
\end{theorem}

The following result was useful in defining the inner product in equation (\ref{eq:inner}), 
where $t_0$ was taken to be $a$.

\begin{theorem} 
\label{thm:A:inner}
Suppose that L is as in (\ref{eq:differentialP}) and let $t_0 \in [a, b]$. Then the only
function in ${\cal{H}}^m[a,b]$ that satisfies L$f =$ the zero function and $f^{(j)}(t_0) = 0, j =
0,\cdots, m - 1$, is the zero function.
\end{theorem}

\begin{proof}
 By Theorem \ref{thm:A:Lnullbasis}, there exists $u_1,\cdots, u_m$ a basis for the set of all $\mu$ with  L$\mu \equiv 0$,
  with
${W}(t)$ invertible for all $t \in [a, b]$. Suppose L$f \equiv 0$. Then $f = \sum_i c_iu_i$ for
some $c_i$'s. We see that the conditions $f^{(j)}(t_0) = 0, ~ j = 0,\cdots, m - 1$ can
be written in matrix/vector form as $(c_1,\cdots, c_m){W}(t_0) = (0,\cdots, 0)$. Since
${W}(t_0)$ is invertible, $c_i = 0, i = 1,\cdots, m$.
\end{proof}

\subsection{The Green's Function Associated with the Differential Operator L}
\label{sec:A:Greensfunction}

Suppose that L is as in (\ref{eq:differentialP}).
The definition below gives the definition of $G(\cdot , \cdot )$, the Green's function associated with 
 L with specified boundary conditions.
Theorem \ref{thm:A:Greens_functiona} gives an explicit  form of $G$. 

\vskip 10pt \noindent
\textbf{Definition}. $G$ is a Green's function for L if and only if
\begin{equation}
\label{eq:Greensint}
f(t) = \int^b_{u=a} G(t, u) ~(\rmL f)(u)~ du
\end{equation}
for all functions $f$ in ${\cal{H}}^m[a,b]$ satisfying the boundary conditions
\begin{equation}
\label{eq:boundary}
f^{(j)}(a) = 0, j = 0,\cdots, m-1. 
\end{equation}
\vskip 10pt

Of course, 
 it's not immediately clear that such a function $G$ exists. However, $G$ exists and is easily calculated using the Wronskian matrix associated with L
 (see Theorem \ref{thm:A:Greens_functiona}). Recall from Theorem \ref{thm:A:Lnullbasis} of Section 
\ref{sec:A:diffeq} that there exists a basis for
the set of all $\mu$ with  L$\mu \equiv 0$, $u_1,\cdots, u_m$, with invertible Wronskian. Furthermore, each $u_i$
has $m$ derivatives. 

\begin{lemma}
 Let $u^{*}_1(t),\cdots, u^{*}_m(t)$ denote the entries in the last row of the
inverse of ${W}(t)$. Then
$u^*_j$ is continuous, $j=1,\ldots,m$.
\end{lemma}
\begin{proof}
The $u^{*}_i$'s are continuous, since $u^{*}_i= (\det {W}(t))^{-1}$ times an expression involving sums and products of $u^{(j)}_l , l = 1,\cdots, m, j = 0,\cdots, m -1$, and
the $u_l$'s have $m-1$ continous derivatives. 
\end{proof}

\begin{theorem}
\label{thm:A:Greens_functiona}
 Let $u^{*}_1(t),\cdots, u^{*}_m(t)$ denote the entries in the last row of the
inverse of ${W}(t)$. Then
\begin{eqnarray}
G(t, u)
& = & \begin{cases}
\sum_{i=1}^m u_i(t) u_i^{*}(u) & \mbox{for $u \leq t$} \\
0 & \mbox{otherwise}
\label{eq:Greens}
\end{cases}
%\label{eq:Greens}
\end{eqnarray}
is a Green's function for L and, for each fixed $t \in [a,b]$, $G(t, \cdot)$ is in $L^2[a,b]$.
\end{theorem}

The following theorem will be useful in the proof of Theorem \ref{thm:A:Greens_functiona}.

\begin{theorem}
\label{thm:A:Greens_functionb}
Let $G$ be as in (\ref{eq:Greens}) and suppose that $h \in {\cal{L}}_2$. If
\[
r(t) = \int^b_a G(t, u) ~h(u) ~du
\]
Then
\begin{equation}
\label{eq:r1}
r \in {\cal{H}}^m[a,b], 
\end{equation}
\begin{equation}
\label{eq:r2}
(\rmL r)(t) = h(t) \quad \text{almost everywhere} ~ t \in [a, b] 
\end{equation}
and
\begin{equation}
\label{eq:r3}
r^{(j)}(a) = 0 \quad j = 0,\cdots, m-1. 
\end{equation}
\end{theorem}
\begin{proof} Write
\[
r(t) = 
\sum_{i=1}^m u_i(t) \int^t_a u^{*}_i(u) ~h(u) ~du
\]
We'll first show that
\begin{equation}
\label{eq:r4}
r^{(j)}(t) =
\sum_{i=1}^m u^{(j)}_i (t) \int^t_a u^{*}_i(u)~ h(u)~ du \quad j = 0,\cdots, m-1 
\end{equation}
and
\begin{equation}
\label{eq:r5}
r^{(m)}(t) = h(t) + \sum_{i=1}^m u^{(m)}_i (t) \int^t_a u^{*}_i(u)~ h(u) ~du \quad 
\text{almost everywhere} ~ t \in [a, b]. 
\end{equation}
These equations follow easily by induction on $j$. We only present the case
$j = 1$. Then
\[
r'(t) = \sum_{i=1}^m u'_i (t)\int^t_a u^{*}_i(u)~ h(u) ~du 
+ \sum_{i=1}^m u_i (t) \frac{d}{dt}\left[\int^t_a u^{*}_i(u)~ h(u) ~du\right]. 
\]
Since $u^{*}_i$ and $h$ are in ${\cal{L}}_2$,
\[
\sum_{i=1}^m u_i(t) \frac{d}{dt} \left[\int^t_a u^{*}_i(u)~ h(u) ~du\right]
 =\sum_{i=1}^m u_i(t)u^{*}_i(t) h(t)
\]
almost everywhere $t$. But, by definition of ${W}$ and the $u^{*}_i$'s, this is equal to
\[
h(t) \sum_i [ {W}(t)]_{i1} [{W}(t) ^{-1}]_{mi} 
= h(t) ~[{W}(t)^{-1} {W}(t)]_{m1} = h(t)~ {\rm{I}} \{ m = 1 \}.
\]
Therefore, for $m = 1$,
(\ref{eq:r5}) holds and for $m > 1$ (\ref{eq:r4}) holds when $j = 1$. For $m > 1$ and $j >1$,
we can calculate derivatives of $r$ up to order $m-1$, and can calculate the
$m$th derivative almost everywhere to prove (\ref{eq:r4}) and (\ref{eq:r5}). Clearly, the $m$th
derivative in (\ref{eq:r5}) is square-integrable. Therefore we've proven (\ref{eq:r1}).

To prove (\ref{eq:r2}), use (\ref{eq:r4}) and (\ref{eq:r5}) and write
\[
(Lr)(t) = r^{(m)}(t) +
\sum^{m -1}_{j=0}\omega_j(t) r^{(j)}(t)
\]

\[
= h(t) + \sum^m_{i=1} u^{(m)}_i (t) \int^t_a u^{*}_i (u) ~h(u)~ du +
\sum^{m -1}_{j=0}\sum^m_{i=1} \omega_j(t) u^{(j)}_i (t) 
\int^t_a u^{*}_i (u) ~h(u)~ du
\]

\[
= h(t) + \sum^m_{i=1} 
  \left[
  u^{(m)}_i (t) + \sum^{m -1}_{j=0}\sum^m_{i=1} \omega_j(t) u^{(j)}_i (t) 
  \right]
  \int^t_a u^{*}_i (u) ~h(u)~ du
\]

\[
= h(t) + \sum_{i=1}^m (\rmL u_i)(t) \int^t_a u^{*}_i (u) ~h(u)~ du = h(t)
\]
since L$u_i \equiv 0$.

Equation (\ref{eq:r3}) follows directly from (\ref{eq:r4}) by taking $t = a$.
\end{proof}

\begin{proof}[Proof of 
Theorem \ref{thm:A:Greens_functiona}]
First consider the function in equation (\ref{eq:Greens}) as a function of $u$ with $t$ fixed.  Since
the $u_i$'s are continuous and $W(u)$ is invertible for all $u$,  $G(t,\cdot)$ is continuous on the
finite closed interval $[a,b]$.  Thus it is in $L^2[a,b]$.  

To show that equation (\ref{eq:Greensint}) holds, 
 let $f \in {\cal{H}}^m$ satisfy the boundary conditions
 (\ref{eq:boundary}). Define $r(t) = \int_a^b G(t, u)~ (\rmL f)(u)~ du$. Then, by Theorem 
 \ref{thm:A:Greens_functionb}, L$r$ = L$f$ almost everywhere and $r^{(j)}(a) = 0,
j = 0,\cdots, m -1$. Thus L$(r-f )$ = 0 almost everywhere and $(r-f)^{(j)}(a) =
0, j = 0,\cdots, m-1$. By Theorem \ref{thm:A:inner}, $r-f$ is the zero function, that is $r = f$ .
\end{proof}

%\section*{Acknowledgements}

%%%%%%%%%%%%%%%%%%%%%%%%%%%%%%%%%%%%%%%%

\end{document}